\documentclass{article}

\usepackage{colt11e}
\usepackage{sidecap}
\usepackage{graphicx}
\usepackage{amsfonts,amssymb}
\usepackage{natbib}
\usepackage{amsmath}
\usepackage{bm}
\usepackage{graphicx}
\usepackage[disable]{todonotes}
\usepackage{algorithm}

\newcommand{\A}{\mathcal{A}}

\newcommand{\PP}{\mathbb{P}}
\newcommand{\EE}{\mathbb{E}}

\newcommand{\EEp}[1]{\mathbb{E}\left[#1\right]}
\newcommand{\R}{\mathbb{R}}
\newcommand{\real}{\mathbb{R}}

\newcommand{\one}[1]{\mathbb{I}_{\{#1\}}}
\newcommand{\bone}{\mathbf{1}}
\newcommand{\bD}{\mathbf{D}}
\newcommand{\bG}{\mathbf{G}}
\newcommand{\bH}{\mathbf{H}}
\newcommand{\bK}{\mathbf{K}}
\newcommand{\bk}{\mathbf{k}}
\newcommand{\bL}{\mathbf{L}}
\newcommand{\bv}{\mathbf{v}}
\newcommand{\eps}{\varepsilon}
\newcommand{\tJ}{\tilde{J}}

\newcommand{\ra}{\rightarrow}
\newcommand{\eqdef}{\stackrel{{\rm def}}{=}}
\newcommand{\be}{\begin{equation}}
\newcommand{\ee}{\end{equation}}

\newcommand{\ie}{\textit{i.e.,}}
\newcommand{\eg}{\textit{e.g.,}}

\newcommand{\LAs}{\underline{N}} %Learner's Actions
\newcommand{\NAs}{\underline{M}} %Nature's Actions

\newcommand{\CumLoss}{L}

\newcommand{\todoa}[2][]{\todo[color=red, #1]{#2}}
\newcommand{\todog}{\todo[color=green]}

\DeclareMathOperator{\im}{Im}
\DeclareMathOperator{\Ker}{Ker}
\DeclareMathOperator{\Int}{Int}

\newtheorem{proposition}{Proposition}
\newtheorem{remark}{Remark}
\newtheorem{defn}{Definition}

\begin{document}

\title{Non-trivial two-armed partial-monitoring games are bandits}
\author{Andr\'as Antos \\
Computer and Automation\\
Research Institute\\
Hungarian Academy of Sciences
\And 
G\'abor Bart\'ok
 \and Csaba Szepesv\'ari\\
Department of Computing Science\\
University of Alberta\\
Edmonton, Canada
}
\date{}
\maketitle

\begin{abstract}
We consider online learning in partial-monitoring games against an oblivious adversary.
We show that when the number of actions available to the learner is two and the game is nontrivial
 then it is reducible to a bandit-like game and thus the minimax regret is $\Theta(\sqrt{T})$.
\end{abstract}

\section{Introduction}
\if0
In this paper we consider the online learning...

\section{Learning in Partial-Monitoring Games}
\fi
%A finite {\em partial-information matrix game} is specified by a pair of $N\times M$ matrices, $(L,H)$,
% where $L$ is called the {\em loss matrix} and $H$ is called the {\em feedback matrix}.
The {\em partial-monitoring games} we consider are defined as follows:
Two players interact with each other in a sequential manner, {\em Learner} and {\em Nature}.
In each time step Learner can choose one of $N$ actions,
 while Nature can choose one of $M$ actions.
We use the notation $\underline{n}=\{1,\ldots,n\}$ for any integer
 and denote the actions of both players by integers, starting from $1$,
 so the action sets are $\LAs$ and $\NAs$.
At the beginning of the game both Learner and Nature are given a pair of $N\times M$ matrices,
 $\bG=(\bL,\bH)$,
 where $\bL$ is the  {\em loss matrix}
 and $\bH$ is the  {\em feedback matrix}.
The elements $\ell_{ij}$ of $\bL$ are real numbers and
 in fact we shall assume that they belong to the $[0,1]$ interval.
The elements $h_{ij}$ of $\bH$ could be chosen from any alphabet.
However, for the sake of simplicity, and without loss of generality (w.l.o.g.),
 we may assume that the elements of $\bH$ are also real numbers.
Now, still at the beginning of the game, Nature
 decides about the sequence of actions $(J_1,J_2,\ldots)$ to be played.
These actions are kept private, \ie\ they are not revealed to Learner.
Nature's actions will also be called \emph{outcomes}.

The game is played in discrete time steps.
At time step $t$ ($t=1,2,\ldots$),
 first Learner chooses an action $I_t$ based on the information available to him up to time $t$.
The choice of the action may be randomized.
Upon announcing his action,
 Learner gets the feedback $h_{I_t,J_t}$ and suffers the loss $\ell_{I_t,J_t}$.
The cycle is then repeated for time step $t+1$.
It is important to note that the loss suffered is not revealed to Learner.

The goal of Learner is to keep his cumulative loss
\[
 \CumLoss_T = \sum_{t=1}^T \ell_{I_t,J_t}
\]
small, where $T$ denotes the time horizon.
Learner's performance is evaluated
 by comparing his cumulative loss to the cumulative loss of the best fixed action from $\LAs$,
\[
 \CumLoss_T^* = \min_{i\in \LAs} \sum_{t=1}^T \ell_{i,J_t},
\]
 giving rise to the {\em cumulative expected regret} (or simply {\em regret}),
\[
 R_T(\A,\bG) = \EEp{\CumLoss_T-\CumLoss_T^*},
\]
 where $\A$ is the {\em strategy} Learner follows.
Note that in the definition of $\CumLoss_T^*$, the best fixed action is selected in hindsight.
When the growth rate of regret is sublinear in $T$,
 \ie\ the average regret $R_T/T$ converges to zero,
 in the long run, Learner can be said to perform as well as an oracle who can play this best action in hindsight.
%In this case Learner is said to be {\em Hannan consistent}.

The problem just described is of major importance in learning theory
 since it models a number of interesting scenarios including
 \emph{apple tasting}~\citep{BaPaSz10-ALT},
 a variant of \emph{label efficient prediction}, %~\citep{CBLuSt06}
 and \emph{dynamic pricing}~\citep{CBLu06:book}.
\todo{Cs: Must add some interesting sounding examples..}
\todog{G: Ez a felsorolas igy eleg szerintem. Nem kene itt peldakat magyarazni ugyis lesz kesobb...}
\todoa{Hol?}
For further discussion and examples see Chapters 2--7 in the book by \citet{CBLu06:book}.

Given a game $\bG=(\bL,\bH)$, our goal is
 to find out the growth rate of the \emph{minimax regret} associated with $\bG$,
 and to design strategies that allow Learner to achieve this minimal growth rate.
Let the worst-case regret of algorithm $\A$ when used in $\bG$ for time horizon $T$ be
\[
 \bar{R}_T(\A,\bG)=\sup_{(J_1,\ldots,J_T)\in\NAs^T} R_T(\A,\bG),
\]
where the supremum is taken over all outcome sequences.
Formally, the {\em minimax regret} of game $\bG$ for time horizon $T$ is defined by
\[
 R_T^*(\bG) = \inf_\A \bar{R}_T(\A,\bG)
 = \inf_\A\sup_{(J_1,\ldots,J_T)\in\NAs^T} R_T(\A,\bG),
\]
where the infimum is taken over all strategies of Learner.
%The minimax regret characterizes the hardness of game $\bG$.
Note that, since for constant outcome sequences $R_T(\A,\bG)\ge 0$,
 also $\bar{R}_T(\A,\bG)\ge 0$ and $R_T^*(\bG)\ge 0$.
\begin{defn}
A game is called {\em trivial} if the minimax regret
 is either $0$ or scales linearly with the number of time steps.
\end{defn}
\begin{lemma}\label{lem:0minimaxregret}
%It is not hard show that
The following three statements are equivalent:\\
a) The minimax regret is zero for each $T$.\\
b) The minimax regret is zero for some $T$.\\
c) There exists an action $i\in\LAs$
 whose loss is not larger than the loss of any other action
 irrespectively of the choice of Nature's action.
%On the other hand, the minimax regret grows linearly with $T$
% when no information is received about Nature's actions. \todoa{This is obscure}
\end{lemma}

The proof is in the Appendix.

\section{Previous work}
The growth rate of the minimax regret is strongly influenced by the choice of $\bL$ and $\bH$.
Consider, for example, the case of  so-called {\em full-information} games,
 where the feedback is sufficient for Learner to recover  Nature's action in each round.
In the simplest case, this is represented by $h_{ij}=j$.
However, from the point of view of the information content of feedback,
 we get an equivalent situation when each row of $\bH$ is composed of pairwise distinct elements.
The following result is known to hold for these games:
\begin{theorem}
\label{thm:fullinfo}
Consider a full-information game $\bG$ when Learner has $N$ actions.
Then there exists an algorithm $\A$ such that for any time horizon $T$,
 $\bar{R}_T(\A,\bG) \le \sqrt{(T/2)\ln N}$.
\end{theorem}

Algorithm $\A$ in the theorem above can be
 the  {\em Exponentially Weighted Average Forecaster} with appropriate tuning
 (see \eg\ \citet[Corollary~4.2]{CBLu06:book}). \todoa{jo reference? ott convex $l$!}\todog{fixed!}
% \citet{LiWa94} calls this the weighted majority algorithm)
% and \cite{Vovk90} (the aggregating algorithm).
%In terms of the lower bound, we will make use of a more general statement.
%The following proposition claims that for any non-trivial game
% there exists a $\Omega(\sqrt{T})$ lower bound on the minimax regret.
%
%In particular, it can be proven by assuming full information
% and applying Theorem~3.7 of \citet{CBLu06:book}.
%\todog{Ennel tobb nem kell sztem}\todoa{nem jo reference! ott $M=2$, $l$ abs.loss, $\exists T$, $N$, es $\sup_{F:|F|=N}$!}
%Note that the regret bounds in Theorem~\ref{thm:fullinfo} and \ref{thm:sqrt_lower} are independent of $M$.
%In fact, these results hold even when $M=\infty$.

Another special case is when the only information that Learner receives is the loss of the action taken
 (\ie\ when $\bH=\bL$), which we call the {\em bandit case}, following \citet{CBLu06:book}.~\todog{footnote torolve}%
%\footnote{``Classically'', bandit problems are defined by the restriction that
% in no round can Learner gain any information about the losses of actions other than the chosen one.
%(This typically corresponds to a game with $M\gg N$.)
%When $\bH=\bL$, depending on $\bL$, this might not be the case;
% our class may contain games which could be ``easier'' than ``classical bandits''.
%Nevertheless, it is easy to see that these games are \emph{at most} as hard as bandit games.
%\todo[inline]{I don't get the last sentence of the footnote (I would remove it).
%Especially, since we have not defined formally ``those'' bandit games.}
%AA: az egesz footnote eleg folosleges: ha a Learner ismeri $l()$-t es $l(I_t,J_t)$-t, akkor nem zarhato ki (van olyan $l()$), hogy kikovetkeztetheti a tobbit. Attol az meg sztem bandita.
%Vagy a classical bandit az, ahol L-t nem ismeri a Learner.(?)
%}
Then, the INF algorithm due to \citet{AuBu09} is known to achieve a constant multiple of the minimax regret:
\begin{theorem}\label{thm:bandit}
Take a bandit game $\bG$ when Learner has $N$ actions.
Then there exists an algorithm $\A$ such that $\bar{R}_T(\A,\bG) \le 15 \sqrt{NT}$.
Further, for any $N$ there exists a game $\bG$ such that
 for any time horizon $T$, $R_T^*(\bG) \ge 1/20 \sqrt{NT}$.
\end{theorem}

The lower bound on the minimax regret is due to \citet{ACFS:2002}
 (also, \citet[Theorem~6.11]{CBLu06:book}),
 while the upper bound is due to \citet{AuBu09}.
(The Exp3 algorithm due to \citet{ACFS:2002} achieves the same upper bound up to logarithmic factors.)

The following theorem, due to \citet{Antos-Bartok-Pal-Szepesvari-2011}, is a lower bound for any non-trivial game.
\begin{theorem}
\label{thm:sqrt_lower}
%There exists a full-information game $\bG$ such that for any $T$,
% $R_T^*(\bG)\ge\sqrt{(T/2)\ln N}$.
If $\bG$ is a non-trivial partial-monitoring game then there exists a constant $c>0$ such that
 for any $T$, $R_T^*(\bG)\ge c\sqrt{T}$.
\end{theorem}
%
%\begin{proposition}
%Given any non-trivial partial-monitoring game $\bG$ there exists a constant $c$ such that
% for any $T$\todoa{jo? igaz? vagy csak $\exists G'$}, $R^*_T(\bG)\ge c\sqrt{T}$.
%\end{proposition}
%
Now, consider the game $\bG = (\bL,\bH)$ with
\be\label{eq:1costly_revealing}
 \bL = \begin{pmatrix} 1 & 1 \\ 0 & 1 \\ 1 & 0 \end{pmatrix}, \qquad
 \bH = \begin{pmatrix} 1 & 2 \\ 1 & 1 \\ 1 & 1 \end{pmatrix}.
\ee
That is, the first action of Learner gives full information about the outcome,
 but it has a high cost,
 while the other two actions do not reveal any information.
Further, the ordering of actions $2$ and $3$ by costs is reversed based on the choice of Nature.
Then, the following holds \citep[Theorem~5.1]{CBLuSt06}:
\begin{theorem}
\label{thm:partial_lower}
The above game has $R_T^*(\bG) = \Omega(T^{2/3})$.
\end{theorem}

This shows that the above game is intrinsically harder than a bandit problem.
Further, the algorithm {\sc FeedExp3} by \citet{piccsch01:discrete}
 is known to achieve this growth-rate \citep{CBLuSt06}:
\begin{theorem}
\label{thm:partial_upper}
Consider any partial-monitoring game $\bG=(\bL,\bH)$
 such that $\bL = \bK \bH$ for some matrix $\bK$.
Then, there exist an algorithm $\A$ such that $\bar{R}_T(\A,\bG) \le C T^{2/3}$,
 where $C$ depends on $N$ and $k^* \eqdef \max_{i,j} |k_{ij}|$.
\end{theorem}

Thus, we see that the difficulty of a game depends on the structure of $\bL$ and $\bH$.
Recently, \cite{BaPaSz10-ALT} classified almost all games by their difficulty
 when the number of actions available to Nature is limited to $M=2$.
In effect, they showed that
 the exponent in the dependence of the minimax regret on $T$ in these games is one of $\{0,1/2,2/3,1\}$.

In this short communication, we investigate the dual case
 when the number of actions available to Nature is not restricted,
 but the number of actions available to Learner is limited to $N=2$.

\section{Result}

In this section we state and prove that, in essence,
 any non-trivial two-action game can be viewed as a bandit game.

%Let $R_T(\A,\bG,J)$ denote the regret of a learning strategy $\A$ on game $\bG=(\bL,\bH)$ after $T$ time steps
% when Nature's action sequence is $J=(J_1,J_2,\ldots)$.
We need some preparations.
First, we will make use of the following concept:
\begin{defn}
Take two games, $\bG=(\bL,\bH)$, $\bG'=(\bL',\bH')$,
 where $\bL$, $\bL'$, $\bH$, and $\bH'$ are $N\times M$ matrices.
We say that $\bG'$ is {\em simulation-and-regret-not-harder} than $\bG$
 (or {\em easier} for short, denoted by $\bG'\le\bG$)
 when the following holds:
Fix any algorithm $\A$.
Then, one can find an algorithm $\A'$ such that
 the behavior of $\A$ on $\bG$ can be replicated by using $\A'$ on $\bG'$ in the sense that
 for the same outcome sequences,
 the two algorithms will choose the same action sequences and
 the regret in the second case is at most the regret in the first case,
 that is, $R_T(\A',\bG')\le R_T(\A,\bG)$.

We say that $\bG$ and $\bG'$ are {\em simulation-and-regret-equivalent}
 (or {\em equivalent}, $\bG'\simeq\bG$) when both $\bG'\le\bG$ and $\bG\le\bG'$.
%Further, the same holds in the other direction, too.
\end{defn}

Clearly, $\le$ is a preorder and $\simeq$ is an equivalence relation on the set of $N\times M$ games,
 moreover, if $\bG'\le\bG$ then $R_T^*(\bG')\le R_T^*(\bG)$,
 and if $\bG\simeq\bG'$ then their minimax regret is the same.
\todo{I don't get this!! What if the losses are all different?
Anyways, I would not call this regret equivalence, but simulation equivalence, or something like this.}
\todoa{And now?}

We need a few simple lemmata on these relations of games:
\begin{lemma}\label{lem:transfL}
%First,
The regret of a sequence of actions in a game does not change
 if the loss matrix is changed
 by subtracting the same real number from each coordinate of one of its columns
 (see \eg\ \citet{piccsch01:discrete}).
Therefore, letting $\bone=(1,\ldots,1)^\top\in\R^N$, $\bv\in\R^M$, and $\bG'=(\bL-\bone\bv^\top,\bH)$,
 we have that $\bG\simeq\bG'$.
\end{lemma}
\begin{lemma}\label{lem:transfH}
If $\bG=(\bL,\bH)$ and $\bG'=(\bL,\bH')$ differ only in their feedback matrices and
 $\bH'$ can be obtained by $h'_{ij} = f_i(h_{ij})$ with the help of some mappings $f_i$ ($i\in\LAs$)
 then $\bG\le\bG'$.
If each $f_i$ is injective then $\bG\simeq\bG'$.
%When the above relation holds between two feedback matrices $\bH$ and $\bH'$
% we shall say that the two feedback matrices are {\em equivalent}.
\end{lemma}

In what follows,
 a transformation of some game into another game that takes either the first or the second form just defined
 shall be called an {\em admissible} transformation.

The following proposition shows that if a 2-armed partial-monitoring game is non-trivial then
 there is no loss in generality by assuming that $\bL=\bK\bH$ for some $\bK\in\real^{2\times2}$.
This statement for arbitrary $N$ and most of the ideas for its proof
 could be extracted from the paper of \citet{piccsch01:discrete}.
An exact detailed proof for $N=2$ is included here for the sake of completeness.
\begin{proposition}\label{prop:K}\todoa{why Proposition?}
Let $\bG_0=(\bL_0,\bH_0)$ be a non-trivial 2-armed partial-monitoring game.
Then, there exist matrices $\bL,\bH\in\real^{2\times M}$ such that
 $\bG_0\le\bG=(\bL,\bH)$ %is equivalent to
 and $\bL=\bK\bH$ for some $\bK\in\real^{2\times2}$.
%Namely, $\bK$ can be as in \eqref{eq:K0011}.
\end{proposition}

\begin{proof}[Proof of Proposition~\ref{prop:K}]
First, we transform $\bL_0$ to $\bL$ using Lemma~\ref{lem:transfL} with $\bv^\top$ being its first row.
Thus, the first row of $\bL$ becomes identically zero,
 and we get a non-trivial game $\bG_1=(\bL,\bH_0)\simeq\bG_0$.
Let $\ell$ \todoa{should be bold!} denote the transpose of the second row of $\bL$.
In what follows we construct the matrix $\bH$ using an admissible transformation of $\bH_0$ defined in Lemma~\ref{lem:transfH}.

We construct matrix $A$ in the following way.
Assume that there are $m_1$ ($m_2$) distinct entries in the first (respectively, second) row of $\bH_0$,
 and transform $\bH_0$ by two injective mappings (Lemma~\ref{lem:transfH}) such that
 the elements of its $i^{\rm th}$ row ($i\in\underline{2}$) are from $\underline{m_i}$.
We define the matrices $A_i\in\real^{m_i\times M}$ as follows:
Let each row of $A_i$ be
 the ``indicator'' row of the corresponding value of the $i^{\rm th}$ row of $\bH_0$,
 that is, $[A_i]_{jk}\eqdef\one{[H_0]_{ik}=j}$.
Define $A$ by stacking these matrices on  top of each other:
 \[
 A=\begin{pmatrix}
     A_1 \\
     A_2 \\
   \end{pmatrix}.
 \]
See Figure~\ref{fig:A} for an example.
\begin{figure}[t]
 \centering
\[
 H_0=
   \begin{pmatrix}
     1 & 2 & 3 & 1 \\
     1 & 2 & 2 & 2
   \end{pmatrix}
   \ \longrightarrow\ A=
   \begin{pmatrix}
     1 & 0 & 0 & 1 \\
     0 & 1 & 0 & 0 \\
     0 & 0 & 1 & 0 \\
     1 & 0 & 0 & 0 \\
     0 & 1 & 1 & 1 \\
   \end{pmatrix}
\]
\caption{\label{fig:A}An example for the construction of matrix $A$ used in the proof of Proposition~\ref{prop:K}.
The first three rows of $A$ are constructed from the first row of $\bH_0$
 which has three distinct elements,
 the remaining two rows are constructed from the second row of $\bH_0$.
For more details, see the text.}
\end{figure}

The following lemma, proven in the Appendix, is key to prove Proposition~\ref{prop:K}.
\begin{lemma}\label{lem:ImA}
%If $\bG_1$ is non-trivial then $\ell\in\im A^\top$.
If $\ell\not\in\im A^\top$ then $\bG_1$ is trivial.
\end{lemma}

%To finish the proof of Proposition~\ref{prop:K}
Using the assumption that $\bG_1$ is non-trivial, we have from Lemma~\ref{lem:ImA} that %~\ref{lem:ImA}
 $\ell\in\im A^\top$ must hold. That is,
 $\ell$ can be written as a linear combination of the rows of $A$:
\[
 \ell=\sum_{i=1}^m\lambda_i\mathbf{a}_i,
\]
where $m=m_1+m_2$ and
 the vectors $\mathbf{a}_i^\top$ are the rows of $A$.
Let
\[%begin{align*}
 \mathbf{h}_1=\sum_{i=1}^{m_1}\lambda_i\mathbf{a}_i \qquad\mbox{and}\qquad
 \mathbf{h}_2=\sum_{i=m_1+1}^{m}\lambda_i\mathbf{a}_i.
\]%end{align*}
Finally, let
\[
 \bH =
 \begin{pmatrix}
  \mathbf{h}_1^\top \\
  \mathbf{h}_2^\top \\
 \end{pmatrix}
\]
and $\bG=(\bL,\bH)$.
Now if the $k^{\rm th}$ and $k'^{\rm th}$ entries of the first row of $\bH_0$ are identical
 then $[\mathbf{a}_i]_k=[\mathbf{a}_i]_{k'}$ for $1\le i\le m_1$,
 hence also $[\mathbf{h}_1]_k=[\mathbf{h}_1]_{k'}$.
The same holds for the second row of $\bH_0$ and $\mathbf{h}_2$.
Thus, $\bH$ can be obtained by appropriate mappings from $\bH_0$,
 and Lemma~\ref{lem:transfH} implies $\bG_1\le\bG$.

On the other hand, setting
\be\label{eq:K0011}
 \bK = \begin{pmatrix}
    0 & 0 \\
    1 & 1 \\
  \end{pmatrix},
\ee
we have that $\bL=\bK\bH$.
\end{proof}

The following Proposition is more than what we need,
 but it is interesting in itself:
\todoa{Igazabol joval egyszerubb lenne, ha kihasznalnank, hogy $\bK$ epp a fenti. De altalaban is erdekes.}
\begin{proposition}\label{prop:Kbandit}\todoa{why Proposition?}
Let $\bG=(\bL,\bH)$ be a 2-armed partial-monitoring game such that
 $\bL=\bK\bH$ for some $\bK\in\real^{2\times2}$.
Then, there exist a $2\times M$ bandit game $\bG'$ such that $\bG\le\bG'$.
If $\bK$ is given by \eqref{eq:K0011} then $\bG\simeq\bG'$.
\end{proposition}

\begin{proof}
We will construct a bandit game $\bG'=(\bL',\bH')\ge\bG$ that satisfies $\bL'=\bH'$.
Let $\bK=[k_{ij}]_{2\times 2}$ and
\[
 \bD=\mathrm{diag}( k_{11}-k_{21},k_{22}-k_{12} )
\]
 be a $2\times 2$ diagonal matrix,
 and define the feedback matrix of $\bG'$ by $\bH'=\bD\bH$.
Then, both rows of $\bH'$ are scalar multiples of the corresponding rows of $\bH$.
Hence, by these mappings and Lemma~\ref{lem:transfH}, $\bG\le(\bL,\bH')$.
% identical elements in a row of $\bH$
% correspond to identical elements in the same row of $\bH'$, that is,
% $\bH'$ can be constructed from $\bH$ by an admissible transformation.
If $\bK$ is given by \eqref{eq:K0011} then $\bD=\mathrm{diag}(-1,1)$,
 thus both mappings are injective and $\bG\simeq(\bL,\bH')$.
On the other hand, $\bK-\bD=\bone\bk^\top$ where $\bk^\top=(k_{21},k_{12})$.
Consider the loss matrix
\[
 \bL'\eqdef\bL-\bone(\bk^\top\bH).
\]
%Since the columns of $\bone\bk^\top\bH$ are a scalar multiple of $\bone$,
By Lemma~\ref{lem:transfL}, $\bG'=(\bL',\bH')\simeq(\bL,\bH')$.
Moreover,
\[
 \bL' = \bL-(\bone\bk^\top)\bH = \bL-(\bK-\bD)\bH = \bD\bH = \bH'.
\]
%i.e., $\bG'$ is a bandit game, thus finishing the proof.
\end{proof}
\begin{remark}
It is worth to consider why the above proof works only for $N=2$.
We used that from any $2\times 2$ matrix $\bK$
 we can subtract a diagonal matrix resulting in a matrix with identical rows.
For $N\ge 3$, this obviously does not hold
 (there is not enough ``degrees of freedom'').
Indeed, for $N\ge 3$,
 we have regret rates between $\Theta(\sqrt{T})$ and $\Theta(T)$,
 for example, Theorem~\ref{thm:partial_lower} and~\ref{thm:partial_upper} show that
 the game in \eqref{eq:1costly_revealing} has minimax regret rate $\Theta(T^{2/3})$. %more general: \cite{BaPaSz10-ALT}.
\end{remark}

Now, we are ready to prove our main result.
\begin{theorem}\label{thm:main}
Each non-trivial 2-armed partial-monitoring game
 is easier than an appropriate $2\times M$ bandit game.
Consequently, its minimax regret is $\Theta(\sqrt{T})$,
 where $T$ is the number of time steps.
\end{theorem}

\begin{proof}
According to Proposition~\ref{prop:K} and~\ref{prop:Kbandit},
 if $\bG_0$ is non-trivial then
%, with the help of the admissible transformations,
 we can construct
 first $\bG=(\bL,\bH)$ such that $\bL=\bK\bH$ and $\bG_0\le\bG$,
 then a $2\times M$ bandit game $\bG'$ such that $\bG\le\bG'$. %equivalent to
Thus $\bG_0\le\bG'$, that implies $R_T^*(\bG_0)\le R_T^*(\bG')=O(\sqrt{T})$ by Theorem~\ref{thm:bandit},
 finishing the proof.
\end{proof}

\appendix
\section*{Appendix}\label{sec:appendix}

\begin{proof}[Proof of Lemma~\ref{lem:0minimaxregret}]
a)$\ra$b) is obvious

b)$\ra$c)
For any $\A$,
\begin{align*}
 \sup_{(J_1,\ldots,J_T)\in\NAs^T} R_T(\A,\bG)
&\ge \sup_{j\in\NAs, J_1=\cdots=J_T=j} \EEp{\CumLoss_T-\CumLoss_T^*}\\
% = \sup_{j\in\NAs, J_1=\cdots=J_T=j} \EEp{\sum_{t=1}^T \ell_{I_t,J_t} - \min_{i\in \LAs} \sum_{t=1}^T \ell_{i,J_t}}
% = \sup_{j\in\NAs} \EEp{\sum_{t=1}^T \ell_{I_t,j} - \min_{i\in \LAs} \sum_{t=1}^T \ell_{ij}}
&= \sup_{j\in\NAs} \EEp{\sum_{t=1}^T \ell_{I_t,j} - T\min_{i\in \LAs} \ell_{ij}}\\
&\ge \sup_{j\in\NAs} \left(\EEp{\ell_{I_1,j}} - \min_{i\in \LAs} \ell_{ij}\right)
 \eqdef f(\A).
\end{align*}
b) leads to
\[
 0 = R_T^*(\bG) = \inf_\A\sup_{(J_1,\ldots,J_T)\in\NAs^T} R_T(\A,\bG) \ge \inf_\A f(\A).
\]
Observe that $f(\A)$ depends on $\A$ through
 only the distribution of $I_1$ on $\LAs$ denoted by $q=q(\A)$ now, %$\in$ probability simplex
 that is, $f(\A)=f'(q)$.
This dependence is continuous on the compact domain of $q$,
 hence the infimum can be replaced by minimum.
Thus $\min_q f'(q)\le 0$, that is, there is a $q$ that for all $j\in\NAs$,
 $\EEp{\ell_{I_1,j}} = \min_{i\in \LAs} \ell_{ij}$.
This implies that the support of $q$ contains only actions
 whose loss is not larger than the loss of any other action
 irrespectively of the choice of Nature's action.

c)$\ra$a)
The algorithm that always plays $i$ has zero regret for all outcome sequences and $T$.
\end{proof}

\begin{proof}[Proof of Lemma~\ref{lem:ImA}]
%To prove by contradiction, we assume that
$\ell\notin\im A^\top$ implies
%\[
 $\langle\ell\rangle\nsubseteq\im A^\top$,
%\]
that is equivalent to
%\[
 $\ell^\bot\nsupseteq\Ker A$,
%\]
 which can be seen by taking the orthogonal complement of both sides
 and using $(\Ker A)^\bot=\im A^\top$.
The latter implies that there exists $v$ \todoa{should be bold!} such that $v\in\Ker A$ but $\ell^\top v\neq0$.
We may assume w.l.o.g. that $\ell^\top v>0$ (otherwise take $-v$).
Note that, since the first $m_1$ rows of $A$ add up to $\bone$ and $v\in\Ker A$,
 the coordinates of $v$ sum to zero.

Let $\Delta^M\subseteq\real^M$ denote the $M$-dimensional probability simplex. \todoa{kell definialni?}
If $p\in\Delta^M$ \todoa{should be bold! also $p_-,p_+,p_0,p_1,p_2$} is a distribution over Nature's actions $\NAs$, %and let $q=Ap$,
 then it is easy to see that the first $m_1$ coordinates of $Ap$
 give the probability distribution of observing the different values of the first row of $\bH_0$
 while Learner chooses action $1$ assuming Nature chooses her actions from $p$.
The same applies to the last $m_2$ coordinates of $Ap$ and action $2$.
It follows that if $Ap_1=Ap_2$ for two distributions then no algorithm can distinguish them.
We find such $p_1$, $p_2$ and apply this idea as follows:

If for all $p\in\Delta^M$, $\ell^\top p\ge 0$ (or $\ell^\top p\le 0$),
 then $\bG_1$ has zero minimax regret and thus it is trivial.
Otherwise, there exist  $p_+$ and $p_-$ with $\ell^\top p_+>0$ and $\ell^\top p_-<0$.
Now either there exists $p_0\in\Int(\Delta^M)$ such that $\ell^\top p_0=0$,
 or we can assume w.l.o.g. that one of $p_+$ and $p_-$ is in $\Int(\Delta^M)$,
 in which case
 there must be again a $p_0\in\Int(\Delta^M)$ on the segment $\overline{p_+ p_-}$
 such that $\ell^\top p_0=0$ by the continuity of $\ell^\top p$ in $p$.
In other words, we have a distribution $p_0$ over $\NAs$ such that
 $p_0$ is not on the boundary of the probability simplex and the expected loss of the two actions are equal.

Now let $p_1=p_0+\eps v$ and $p_2=p_0-\eps v$ for some $\eps>0$.
If $\eps$ is small enough then both $p_1$ and $p_2$ are on the probability simplex $\Delta^M$.
Since $Av=0$ we have that $Ap_1=Ap_2$.
%and thus these two distributions are indistinguishable by any algorithm.
%On the other hand, $\ell^\top p_1>0$ and $\ell^\top p_2<0$ imply that
% the optimal action is different under the two distributions over $\NAs$,
% yielding linear minimax regret.

%For proving the lower bound
Given a $p\in\Delta^M$, we use randomization such that $J_1,\ldots,J_T$
 is replaced by a vector $\tJ_1,\ldots,\tJ_T\in\NAs^T$
 of i.i.d.\ random variables distributed according to $p$,
 independent of the randomization in the algorithm.
%By constructing a lower bound on the expected regret
% we also show that there exists an outcome sequence for which the same lower bound holds.
Let $\A$ be an arbitrary strategy of Learner.
For $k\in\underline{2}$,
 given that the outcome distribution is $p_k$,
 let $\PP_k[\cdot]$ be the probability of an event and
 $\EE_k[\cdot]$ be the expectation of a random variable.
Then %it is easy to see that
 the worst case regret of $\A$ is
 {\allowdisplaybreaks\todog{allowdisplaybreaks}
\begin{align*}
 \sup_{(J_1,\ldots,J_T)\in\NAs^T} R_T(\A,\bG_1)
&\ge \EE_k[R_T(\A,\bG_1)] %= \EE_k[\CumLoss_T-\CumLoss_T^*]
\\
&= \EE_k\left[ \sum_{t=1}^T \ell_{I_t,\tJ_t}-\min_{i\in\underline{2}} \sum_{t=1}^T \ell_{i,\tJ_t} \right]\\
&= \EE_k\left[ \sum_{t=1}^T \one{I_t=2} \ell_{\tJ_t}
 - \min \left(\sum_{t=1}^T \ell_{\tJ_t},0\right) \right]\\
& \mbox{($\ell_{1j}=0$, $\ell_{2j}=\ell_j$)}\\
&\ge \sum_{t=1}^T \EE_k\left[ \one{I_t=2} \right] \EE_k\ell_{\tJ_t}
 - \min \left(\sum_{t=1}^T \EE_k\ell_{\tJ_t},0\right)\\
& \mbox{(by the independence of $I_t$ and $\tJ_t$,
 and Jensen's inequality for $\min$)}\\
&= \sum_{t=1}^T \PP_k[I_t=2] \ell^\top p_k
 + \left(-\sum_{t=1}^T \ell^\top p_k\right)^+\\
%&= \ell^\top p_k \sum_{t=1}^T \PP_k[I_t=2] + T\left(-\ell^\top p_k\right)^+\\
&= \ell^\top p_k\mu_{Tk}+ T(-\ell^\top p_k)^+,
\end{align*}}
where
\[
 \mu_{Tk}=\mu_{Tk}(\A)\eqdef\sum_{t=1}^T \PP_k[I_t=2]\in[0,T]
\]
 is the expected number of times $\A$ chooses action $2$ under $p_k$ up to time $T$.
Observe that $Ap_1=Ap_2$ means that for both actions,
 the feedback distribution is the same under outcome distributions $p_1$ and $p_2$,
 implying (by induction) that for each $t\ge 1$, $\PP_1[I_t=2]=\PP_2[I_t=2]$.
This leads to $\mu_{T1}=\mu_{T2}\eqdef\mu_T=\mu_T(\A)$.
Moreover, using $\ell^\top p_0=0$ and $\ell^\top v>0$,
\[
\ell^\top p_k\mu_T + T(-\ell^\top p_k)^+ =
\left\{\begin{array}{ll}
  \eps \ell^\top v \mu_T & \mbox{if $k=1$},\\
  \eps \ell^\top v (T-\mu_T) & \mbox{if $k=2$}.
 \end{array} \right.
\]
Thus we have
\begin{align*}
 R_T^*(\bG_1)
&= \inf_\A\sup_{(J_1,\ldots,J_T)\in\NAs^T} R_T(\A,\bG_1)
 \ge \inf_\A\max_{k\in\underline{2}} (\ell^\top p_k\mu_T + T(-\ell^\top p_k)^+)\\
% = \inf_\A\max(\eps\ell^\top v\mu_T,\eps\ell^\top v(T-\mu_T))
&= \eps \ell^\top v \inf_\A\max(\mu_T,T-\mu_T)
 \ge \eps \ell^\top vT/2,
\end{align*}
that is, $\bG_1$ is trivial.
\end{proof}

\bibliography{allbib,confs,biblio}
\bibliographystyle{plainnat}

\end{document}